\documentclass{article}
\usepackage{graphicx}
\usepackage{subfigure}
\usepackage{enumitem}
\usepackage{dsfont}
\usepackage{amsmath}
\usepackage{amsfonts}
\usepackage{amsthm}
\usepackage{fullpage}
\usepackage{algorithm}
\usepackage{algorithmic}
\usepackage{hyperref}
\newtheorem{df}{Definition}[section]
\newtheorem{assum}{Assumption}[section]
\newtheorem{lemma}{lemma}[section]
\newtheorem{theorem}{Theorem}[section]
\newtheorem*{trm}{Theorem}
\newtheorem*{prf}{Proof}
\title{Graph Approximation and Clustering on a Budget}
\author{
Ethan Fetaya \\
Weizmann Institute of Science\\
\texttt{ethan.fetaya@weizmann.ac.il} \\
\\
Ohad Shamir  \\
Weizmann Institute of Science \\
\texttt{ohad.shamir@weizmann.ac.il} \\
\\
Shimon Ullman \\
Weizmann Institute of Science \\
\texttt{shimon.ullman@weizmann.ac.il}
}
\begin{document}

\date{}
\maketitle

\begin{abstract}
We consider the problem of learning from a similarity matrix (such as
spectral clustering and low-dimensional embedding), when computing
pairwise similarities are costly, and only a limited number of entries can
be observed. We provide a theoretical analysis using standard notions of
graph approximation, significantly generalizing previous results (which
focused on spectral clustering with  two clusters). We also propose a new
algorithmic approach based on adaptive sampling, which experimentally
matches or improves on previous methods, while being considerably more
general and computationally cheaper.
\end{abstract}

\section{Introduction}

Many unsupervised learning algorithms, such as spectral clustering
\cite{normCut}, \cite{specClus} and low-dimensional embedding via Laplacian
eigenmaps and diffusion maps \cite{LapEigMap},\cite{DiffMaps}, need as input
a \emph{matrix of pairwise similarities} $W$ between the different objects in
our data. In some cases, obtaining the full matrix can be a costly matter.
For example, $w_{ij}$ may be based on some expensive-to-compute metric such
as W2D \cite{W2D} ; based on some physical measurement (such as in some
computational biology applications); or is given by a human annotator. In
such cases, we would like to have a good approximation of the (initially
unknown) matrix, while querying only a limited number of entries. An
alternative but equivalent viewpoint is the problem of approximating an
unknown weighted undirected graph, by querying a limited number of edges.

This question has received previous attention in works such as \cite{SCoaB}
and \cite{MJ24}, which focus on the task of spectral clustering into two
clusters, and assuming two such distinct clusters indeed exist (i.e. that there
is a big gap between the second and third eigenvalues of the Laplacian
matrix). In this work we consider, both theoretically and algorithmically,
the question of query-based graph approximation more generally, obtaining
results relevant beyond two clusters and beyond spectral clustering.

When considering graph approximations, the first question is what notion of
approximation to consider. One important notion is \emph{cut approximation}
\cite{KarSpar} where we wish for every cut in the approximated graph to have
weight close to the weight of the cut in the original graph up to a
multiplicative factor. Many machine learning algorithms (and many more
general algorithms) such as cut based clustering  \cite{cutClus},  energy
minimization \cite{energy}, etc. \cite{cutsVision}
are based on cuts, so this notion of approximation is natural for these uses.
A stronger notion is \emph{spectral approximation} \cite{speilman}, where we
wish to uniformly approximate the quadratic form defined by the Laplacian up
to a multiplicative factor. This approximation is important for algorithms
such as spectral clustering \cite{specClus}, Laplacian eigenmaps
\cite{LapEigMap}, diffusion maps \cite{DiffMaps}, etc. that use the
connection between the spectral properties of the Laplacian matrix and the
graph.

Our theoretical analysis focuses on the number of queries needed for such
approximations. We first consider the simple and intuitive strategy of
sampling edges uniformly at random, and obtain results for both cut and
spectral approximations, under various assumptions.  We note that these
results are considerably more general than the theoretical analysis in
\cite{SCoaB}, which focuses on the behavior of the 2nd eigenvector of the
Laplacian matrix, and crucially rely on this large eigengap. We then consider
how to extend these results to adaptive sampling strategies, and design a
generic framework as well as a new adaptive sampling algorithm for clustering
(\texttt{CLUS2K}). Compared to previous approaches, the algorithm is much
simpler and avoids doing a costly full eigen-decomposition at each iteration,
yet experimentally appears to obtain equal or even better performance on a
range of datasets.

Our theoretical results build on techniques for graph sparsification
(\cite{KarSpar}, \cite{speilman}), where the task is to find a sparse
approximation to a given graph $G$. This is somewhat similar to our task, but
with two important differences: First and foremost, we do not have access to
the full graph, whereas in graph sparsification the graph is given, and this
full knowledge is used by algorithms for this task (e.g. using the sum of
edge weights associated with each node). Second, our goal is to minimize the
number of edge sampled, not the number of edges in the resulting graph (of
course by sampling a smaller number of edges we will get a sparser graph).
Notice that if we wish to end with a sparse graph, one can always use any
graph sparsification technique on our resulting graph and get a graph with
guarantees on sparsity.

%We will focus first on uniform sampling and show that if a graph is dense,
%in the sense that his smallest cut/eigenvalue is not too small, it can be
%approximated. Next we will show that if a graph is "clusterable", i.e. is partitioned into loosely connected dense clusters, then we can recover these clusters after enough samples (that depend on the density of these clusters).

%Our goal is to build on the work in \cite{SCoaB} and
%We will show how random sampling and adaptive sampling techniques can be used
%for graph approximation in various scenarios. We will also present a new
%adaptive sampling algorithm for clustering, the \texttt{CLUS2K} algorithm, that
%outperforms current algorithms while avoids doing a costly full
%eigen-decomposition at each step.

%We will then show that (under stricter requirements) one can achieve a good spectral approximation, and how in this case the same theoretical guarantees that hold for uniform sampling, hold for adaptive sampling. Last we will show a fast and simple adaptive sampling algorithm that out performs previous works.

\section{A General Graph Approximation Guarantee}
For simplicity we will consider all graphs as full weighted graphs (with zero
weights at edges that do not exist) and so any graph will be defined by a set
of vertices $V$ and a weight matrix $W$. We will start with a few basic
definitions.
\begin{df}
Let $G=(V,W)$ be a weighted graph and $S\subset V$ a subset of vertices, then
the cut defined by $S$, $|\partial_G S|$, is the sum of all the weights of
edges that have exactly one endpoint in $S$.
\end{df}

\begin{df}
Let $G=(V,W)$ and $\tilde{G}=(V,\tilde{W})$ be two graphs on the same set of
vertices. $\tilde{G}$ is an $\epsilon$-cut approximation  of $G$
if for any $S\subset V$ we have $(1-\epsilon)|\partial_G S| \leq
|\partial_{\tilde{G}} S|\leq (1+\epsilon)|\partial_G S|$
\end{df}

\begin{df}\label{LapDef}
Let $G=(V,W)$. The graph Laplacian $L_G$ is
defined as $L_G=D-W$ where $D$ is a diagonal matrix with values $D_{ii} =
\sum\limits_{1\leq j \leq n} W_{ij}$. The normalized graph Laplacian $\mathcal{L}_G$ is
defined as $\mathcal{L}_G=D^{-1/2}(D-W)D^{-1/2}=D^{-1/2}L_G D^{-1/2}$.
\end{df}
The Laplacian holds much information
about the graph \cite{Chung}. One of the main connections of the Laplacian,
and in particular the quadratic form it represents, to the graph is through
the equation
\begin{equation}\label{quadLap}
x^TL_Gx = \frac{1}{2}\sum_{i,j=1}^nW_{ij}(x_i-x_j)^2
\end{equation}
When $x_i\in\{0,1\}$ this is easily seen to be the value of the cut
defined by $x$. Many spectral graph techniques, such as spectral
clustering, can be seen as a relaxation of such a discrete problem to $x\in\mathds{R}^n$.

\begin{df}
A graph $\tilde{G}$ is an $\epsilon-$spectral approximation of $G$ if
\begin{equation}\label{specApprox}
\forall x\in\mathds{R}^n\quad (1-\epsilon)x^TL_{\tilde{G}}x\leq x^TL_Gx \leq (1+\epsilon)x^TL_{\tilde{G}}x
\end{equation}
\end{df}
We note that this is different than requiring
$||L_G-L_{\tilde{G}}||\leq\epsilon$ (using the matrix 2-norm) as we can view
it as a multiplicative error vs. an additive error term. In particular, it
implies approximation of eigenvectors (using the min-max theorem
\cite{Chung}), which is relevant to many spectral algorithms, and includes
the approximation of the 2nd eigenvector (the focus of the analysis in
\cite{SCoaB}) as a special case Moreover, it implies cut approximation (via equation \ref{quadLap}),
and is in fact strictly stronger (see \cite{speilman} for a simple example of
a cut approximation which is not a spectral approximation). Therefore, we
will focus on spectral approximation in our theoretical results.

Our initial approximation strategy will be to uniformly at random sample a
subset $\tilde{E}$ of $m$ edges, i.e. pick $m$ edges without replacement and
construct a graph $\tilde{G}=(V,\tilde{W})$ with weights $\tilde{w}_{ji}
=\tilde{w}_{ij} = \frac{w_{ij}}{p}$  for any $(i,j)\in\tilde{E}$ and zero
otherwise, where $p= {m}/{\tbinom{n}{2}}$ is the probability any edge is
sampled. It is easy to see that the $\mathbb{E}[\tilde{W}]=W$.

We begin by providing a bound on $m$ which ensures an $\epsilon$-spectral
approximation. It is based on an adaptation of the work in \cite{speilman}, in
which the author considered picking each edge independently.
This differs from our setting, where we are interested in picking $m$ edges
without replacement, since in this case the probabilities of picking
different edges are no longer independent. While this seems like a serious
complication, it can be fixed using the notion of negative dependence:
\begin{df}
The random variables $X_1,..., X_n$ are said to be \emph{negatively
dependent} if for all disjoint subset $I,J\subset [n]$ and all nondecreasing
functions $f$ and $g$,
\begin{equation*}
\mathds{E}[f(X_i, i\in I)g(X_j, j\in J)]\leq \mathds{E}[f(X_i, i\in I)]\mathds{E}[g(X_j, j\in J)]
\end{equation*}
\end{df}
Intuitively, a group of random variables are negatively dependent if when some
of them have a high value, the others are more probable to have lower values.
If we pick $m$ edges uniformly, each edge that has been picked lowers the
chances of the other edges to get picked, so intuitively the probabilities
are negatively dependent.  The edge picking probabilities  have  been shown to be indeed negatively dependent in \cite{SCoaB}.

An important application of negative dependence is that the
Chernoff-Hoeffding bounds, which hold for sums of independent random
variables, also hold for negatively dependent variables.
See  supplementary material for details.

We can now state the general spectral approximation theorem:

\begin{theorem}\label{spectralBound}
Let $G$ be a graph with weights $w_{ij}\in[0,1]$ and $\tilde{G}$ its
approximation after sampling $m$ edges uniformly. Define $\lambda$ as the
second smallest eigenvalue of  $\mathcal{L}_G$ and
$k=\max\{\log(\frac{3}{\delta}),\log(n)\}$.  If $m\geq\binom{n}{2}
\frac{\left(\frac{12k}{\epsilon \lambda}\right)^2}{\min D_{ii}}$ then the
probability that $\tilde{G}$ is not an $\epsilon$-spectral approximation is
smaller then $\delta$.
\end{theorem}

\begin{proof}[Proof sketch]
The proof is based on an adaptation of part of theorem 6.1 from
\cite{speilman}. The two main  differences are that we use negative
dependence instead of independence, and a weighted graph instead of an unweighted
graph. The proof uses the following  lemma

\begin{lemma}\label{speilmanLemma}
Let $\mathcal{L}_G$ be the normalized Laplacian of $G$ with second eigenvalue
$\lambda$. If $||D^{-1/2}(L_G-L_{\tilde{G}})D^{-1/2}||\leq \epsilon$ then
$\tilde{G}$ is an $\sigma$-spectral approximation for
$\sigma=\frac{\epsilon}{\lambda-\epsilon}$.
\end{lemma}

The next part is to bound $||D^{-1/2}(L_G-L_{\tilde{G}})D^{-1/2}||$ using a
modified version of the trace method \cite{Vu} in order to bound this norm.
See the supplementary material for more details.
\end{proof}
Stating the result in a simplified form, we have that if $\min D_{ii} =
\Omega(n^\alpha)$, then one gets an $\epsilon$-approximation guarantee using
$m=\mathcal{O}\left(n^{2-\alpha}\left(
\frac{\log(n)+\log(1/\delta)}{\epsilon\lambda}\right)^2\right)$ sampled edges.

The main caveat of theorem \ref{spectralBound} is that it only leads to a non-trivial guarantee $(m\ll n^2)$ when $\alpha>0$ and
$\lambda$ is not too small. Most algorithms, such as spectral clustering, assume that the graph has $k\geq2$ relatively small eigenvalues, in the ideal case (more then one connected component) we even have $\lambda=0$ . %Before we try to resolve this, we will show next that there is a fundamental difficulty in approximating certain graphs. The following will also make the dependence on $\min D_{ii} $ clearer.  %\begin{lemma} \label{combLemma}
%Let $G$ be an undirected graph with $n$ vertices and minimal cut $c>0$. %For all $\alpha\geq1$ the number of cuts with weight smaller of equal to %$\alpha c$ is less then $n^{2\alpha}$.
%\end{lemma}
%The lemma is proven in \cite{KarLemma} for graphs with integer weights, %but the extension to any positive weights is trivial by scaling and rounding.
Unfortunately, we will now show that this is unavoidable, and that the bound above is essentially optimal (up to log factors) for graphs with bounded $\lambda>C>0$, i.e. expanders. In the next section, we will show how a few reasonable assumptions allows us to recover non-trivial guarantees even in the regime of small eigenvalues.
%\section{Lower Bound on the Number of Samples}\label{lowerBound}

Since spectral approximation implies cut approximation, we will use this to
find simple bounds on the number of edges needed for both approximations. We
will show that a necessary condition for any approximation is that the
minimal cut is not too small, the intuition being that that even finding a single edge (for connectedness) on that cut can be hard, and get a lower bound on the number of samples
needed.  For this we will need to following simple lemma (which follows
directly from the linearity of expectation)

\begin{lemma}\label{interLemma}
Let $X$ be a finite set, and $Y\subset X$. If we pick a subset $Z$ of size
$m$ uniformly at random then $\mathds{E}\left [ |Z\cap Y|\right]=\frac{m\cdot
|Y|}{|X|}$
\end{lemma}

We will now use this to prove a lower bound on the number of edges sampled
for binary weighted graphs (i.e. unweighted graph) $w_{ij}\in \{0,1\}$ .

\begin{theorem}\label{KarLower}
Let G be an binary weighted graph with minimal cut weight $c$$>$$0$. Assume
$\tilde{G}$ was constructed by sampling
$m$$<$$\tbinom{n}{2}\frac{(1-\delta)}{c}$ edges, then for any
$\epsilon$$<$$1$, $P\left( \tilde{G}~\text{is not an $\epsilon$-cut approximation of $G$}\right)>\delta$.
\end{theorem}
\begin{proof}
Let $Y$ be all the edges in a minimal cut and let $\tilde{E}$ be the edges
sampled. Since the weights are binary the weight of this cut in $\tilde{G}$
is the number of edges in $Y\cap \tilde{E}$. From lemma \ref{interLemma} we
know that $\mathds{E}\left [ |Y\cap \tilde{E}|\right]=mc /
\tbinom{n}{2}<1-\delta$. From the Markov's inequality we get that
$P\left(|Y\cap \tilde{E}|\geq 1\right)<1-\delta$. If $|Y\cap \tilde{E}|<1$
then the intersection is empty and we do not have an $\epsilon$-approximation
for any $\epsilon<1$ proving $P\left( \tilde{G}~\text{is an $\epsilon$-cut approximation of $G$}\right)<1-\delta$.
\end{proof}

This theorem proves that in order to get any reasonable approximation with a
small budget $m$ (at least with uniform sampling) the original graphs minimal
cut cannot be too small and that $\Omega({n^2}/{c})$ samples are needed.
Comparing this to theorem \ref{spectralBound} (noticing $\min_iD_{ii}\geq c$) we see that, for graphs with a lower bound on $\lambda,$ by sampling within a logarithmic
factor of this lower bound is sufficient to ensure a good cut approximation.

\section{Clusterable Graphs}
Clustering algorithms assume a certain structure of the graph, generally they assume $k$ strongly connected components, i.e. the clusters, with weak connections between them (the precise assumptions vary from algorithm to algorithm). While this is a bad scenario for approximation, as this normally means a small minimal cut (and for spectral clustering a small $\lambda$), we will show how approximation can be used (on the inner-cluster graphs) to obtain useful results. We give two results, one geared towards spectral approximation, and the other towards cut approximation.

\subsection{Spectral Approximation}

\begin{df}
Assume a graph $G=(V,W)$ consists of $k$ clusters, define $W^{in}$ as the
block diagonal matrix consisting of the similarity scores between
same-cluster elements.
 \begin{equation*}
\mathbf{W^{in}} = \begin{bmatrix}
\mathbf{W}^{1} & \mathbf{0} & \cdots &\mathbf{0}\\
\mathbf{0} & \mathbf{W}^{2} & \cdots &\mathbf{0}\\
\vdots          & \vdots          & \ddots &\vdots \\
\mathbf{0} & \mathbf{0} & \cdots &\mathbf{W}^{k}\end{bmatrix}
\end{equation*}
and $W^{out}=W-W^{in}$  the off-diagonal elements.\\
\end{df}
\begin{assum}\label{assumLambda}
Define $\lambda^{in}=\min\limits_{1\leq i\leq k}\lambda_2(\mathcal{L}^i)$,
the smallest of all the second normalized eigenvalues over all $\mathcal{L}^i=\mathcal{L}_{W^i}$. Assume
$\lambda^{in}>C>0$ for some constant $C$.
\end{assum}
\begin{assum}\label{assumD}
$\min D^{in}_{ii} = \Omega(n^\alpha)$ where $D^{in}_{ii} = \sum_j
W^{in}_{ij}$ for some $\alpha>0$.
\end{assum}
\begin{assum}\label{outsmall}
Assume that $||W^{out}|| = \mathcal{O}(n^\beta)$ for some $\beta<\alpha$.
\end{assum}
Assumption \ref{assumLambda} implies well connected clusters, while
assumption \ref{assumD} excludes sparse, well-connected graphs, which we have
already shown earlier to be hard to approximate.
Assumption \ref{outsmall} essentially requires the between-cluster
connections to be relatively weaker than the within-cluster connections.

Under these assumptions, the next theorem proves the clusters can be found.
\begin{theorem}\label{specClusterable}
Let P be the zero eigenspace of $L^{in}=D^{in}-W^{in}$ corresponding to the
$k$ clusters, $\tilde{L}$ the Laplacian of the graph we get by sampling
$m=\tilde{\mathcal{O}}(n^{2-\gamma})$ edges for $\beta\leq\gamma\leq\alpha$,
and $Q$ the space spanned by the first $k$ eigenvectors of $\tilde{L}$. Under
previous assumptions, $||\sin\left(\Theta\left(P,Q\right)\right)|| =
\mathcal{O}(\frac{n^\beta+n^\gamma}{n^{\alpha}})$.
\end{theorem}

We simplified the statement in order not to get overwhelmed by notation.
$\Theta\left(P,Q\right)$ is a diagonal matrix whose diagonal values
correspond to the canonical angles between the subspaces $P$ and $Q$, and
$||\sin\left(\Theta\left(P,Q\right)\right)||$ is a common way to measure
distance between subspaces.

\begin{proof}[Proof sketch]
If $\gamma=0$, i.e. $Q$ was spanned by eigenvectors of the full $L$, then the
theorem would be true by the sin-theta theorem \cite{sin_theta} using our
assumptions. We need to show that this theorem can be used with $\tilde{L}$.
The sin-theta theorem states that
$||\sin\left(\Theta\left(P,Q\right)\right)||\leq\frac{||\tilde{L}^{out}||}{\mu_2}$
where $||\tilde{L}^{out}||$ the "noise" factor, and $\tilde{\mu}_2$ the
\underline{unnormalized} second eigenvalue of $\tilde{L}^{in} $ the "signal"
factor. Using theorem \ref{spectralBound} and our first two assumptions we
can approximate each $L_{W^i}$ and use to show that
$\tilde{\mu}_2=\Omega(n^\alpha)$. We now only need to show
$||\tilde{L}^{out}||=\mathcal{O}(n^\beta+n^{\gamma})$. This can be done using
the matrix Chernoff  inequality \cite{tropp}, by applying a result in \cite{sampWithout} that shows how it can be adapted to sampling without replacements. We note that the  result in \cite{sampWithout} is limited to sampling without replacements as negative dependence has no obvious extension to random matrices. For further details see the
supplementary material
\end{proof}

This gives us a tradeoff between the number of edges sampled and the error.
The theoretical guarantee from the sin-theta theorem for the complete graph
is $\mathcal{O}(n^\beta/n^{\alpha})$ so for $\gamma=\beta$ we have the same
guarantee as if we used the full graph. For $n$ large enough one can get
$||\sin\left(\Theta\left(P,Q\right)\right)||$ as small as desired by using
$\gamma=\alpha-\epsilon$.
\subsection{Cut Approximation}
Cut based clustering, such as \cite{cutClus}, have  a different natural notion of "clusterable". We will assume nothing on eigenvalues, making this more general than the previous section.
\begin{assum}\label{assCut1}
Assume $G$ can be partitioned into $k$ clusters, within which the minimal cut
is at least $c_{in}$. Furthermore, assume that any cut separating between the
clusters of $G$ (but does not split same cluster elements) is smaller then
$c_{out}$, and that $c_{in}>4 c_{out}$.
\end{assum}

These  assumptions basically require the inner-cluster connections to be relatively stronger than between-cluster connections.
\begin{theorem}\label{clustering}
Let $G$ be a graph with weights $w_{ij}\in[0,1]$ and $\tilde{G}$ its approximation after observing $m$ edges. Under previous assumptions if $m=\tilde\Omega\left(\frac{n^2}{c_{in}}k\ln(\frac{1}{\delta})\right)$ then the cuts separating the clusters are smaller then any cut that cuts into one of the clusters.
\end{theorem}
\begin{proof}[Proof Sketch]
We can use  cut approximation  for the clusters themselves so $\tilde{c}_{in}\geq c_{in}/2$. Using the Chernoff bound and union bound for the $2^k$ cuts between clusters, we get that none of them is greater then $c_{in}/2$.
See the supplementary material for full proof.\end{proof}

In the supplementary material, we provide a more in-depth analysis of cut approximation  including an analog of theorem \ref{spectralBound}.
\section{Adaptive Sampling and the Clus2K Algorithm}

Theorem \ref{KarLower} states that, with uniform sampling and no prior assumptions of the graph structure, we need at least
$\Omega(n^2/c)$ where $c$ is the weight of the smallest cut. What if we had
an adaptive algorithm instead of just uniform sampling? It is easy to see
that for some graphs the same (up to a constant) lower bound holds. Think of
a graph with $2n$ vertices, consisting of two cliques that have c randomly
chosen edges connecting them. Let's assume further that some oracle told us
which vertex is in which clique, so any sensible algorithm would sample only
edges connecting the cliques. As the edges are random, it would take
$\Theta\left( {n^2}/c \right)$ tries just to hit one edge needed for any good
approximation. However, in some cases an adaptive scheme can reduce the
number of samples, as we now turn to discuss in the context of clustering.

Consider a similar toy problem - we have a graph which is known to consist of two connected
components, each a  clique of size $n$ and we wish to find these clusters. We can run the
uniform sampling algorithm until we have only two connected components and
return them. How many edges do we need to
sample until we get only two connected components? If we look only at one
clique, then basic results in random graph theory \cite{randomGraph} show
that with high probability, the number of edges added before we get a
connected graph is $\Theta(n\log(n))$ which lower bounds the number of
samples needed. To improve on this we can use an adaptive algorithm with the following scheme:
at each iteration, pick an edge at random connecting the smallest connected
component to some other connected component.  At each step we have at
least a probability of $\frac{1}{3}$ to connect two connected components.
This is because there are $n$ nodes in the wrong cluster, and at least
$\frac{n}{2}$ in the right cluster (since we pick the smallest connected
component). Therefore with high probability the number of steps needed to
decrease the number of connected components from $2n$ to two is $\Theta(n)$.

This argument leads us to consider adaptive sampling schemes, which
iteratively sample edges according to a non-uniform distribution.
Intuitively, such a distribution should place more weight on edges which may
be more helpful in approximating the structure of the original graph. We
first discuss how we can incorporate \emph{arbitrary} non-uniform
distributions into our framework. We then propose a specific
non-uniform distribution, motivated by the toy example above, leading to a
new algorithm for our setting in the context of clustering.

One approach to incorporate non-uniform distributions is by unbiased
sampling, where we re-scale the weights according to the sampling
probability. This means that the weights are unbiased estimates of the actual
weights. Unfortunately, this re-scaling is not easy to compute in general
when sampling without replacement, as the probability of sampling an edge  is a marginal distribution over all the algorithm's
possible trajectories. Sampling with replacement is much easier, since it
only depends on the sampling probability in the current iteration. Moreover,
as long as we sample only a small part of all edges, the risk of re-sampling
an already-sampled edge is negligible. Finally, one can show that whatever
the non-uniform distribution, a simple modification (adding with probability
half a uniform sample) suffices for cut approximation to hold. Unfortunately, we found this approach to work poorly in practice, as it was unstable and oscillated between good and bad clustering long after a good clustering is initially found. 

Due to these issues, we consider a \emph{biased} sampling without replacement approach, where we
mix the non-uniform distribution with a uniform distribution (as proposed
earlier) on unseen edges, but do not attempt to re-scale of weights. More specifically,
consider any adaptive sampling algorithm which picks an unseen edge at step
$i+1$ with probability $p(e;\tilde{G}_i)$ that depends on the graph
$\tilde{G}_i$ seen so far.  We will consider a modified distribution that
with probability $0.5$ picks an unseen edge uniformly, and with probability
$0.5$ picks it according to $p(e;\tilde{G}_i)$.
 While biased sampling can ruin approximation guarantees, in the clustering scenarios one can show similar results to theorem \ref{clustering} (under stronger conditions) for any adaptive sampling scheme. The theoretical guarantees for adaptive sampling are in the supplementary material. 
\subsection{\texttt{CLUS2K} Algorithm}

We now turn to consider a specific algorithmic instantiation, in the context
of clustering. Motivated by the toy example from earlier, we consider a
non-uniform distribution which iteratively attempts to connect clusters in
the currently-observed graph, by picking edges between them. These clusters
are determined by the clustering algorithm we wish to use on the approximated graph, and are incrementally
updated after each iteration. Inspired by the common practice (in computer
vision) of over-segmentation, we use more clusters than the desired number of
clusters $k$ ($2k$ in our case). Moreover, as discussed earlier, we mix this
distribution with a uniform distribution. The resulting algorithm, which we
denote as \texttt{CLUS2K}, appears as Algorithm \ref{alg:example} below.

\begin{algorithm}[!ht]
   \caption{\texttt{CLUS2K}}
   \label{alg:example}
\begin{algorithmic}
   \STATE {\bfseries Input:} budget $b$, number of clusters $k$\\
   \STATE {\bfseries Initialize:} $S=\{(i,j)\in\{1,...,n\}^2:\,i<j\}$, $\tilde{W}$ the zero matrix.
   \FOR{$t=1,...,b$}
        \STATE With probability 1/2 pick $(i,j)\in S$ uniformly;
        \STATE Otherwise:
        \STATE \qquad $C_1,...,C_{2k}\leftarrow$cluster $\tilde{W}$ into $2k$ clusters;\\
        \STATE \qquad pick two distinct clusters $C_l$ and $C_m$ uniformly at random;\\
        \STATE \qquad pick $(i,j)\in S$ connecting $C_l$ and $C_m$ uniformly at random;\\
        \STATE Set $\tilde{w}_{ij} = \tilde{w}_{j,i}= w_{ij}; S = S\backslash (i,j);$
   \ENDFOR
 
\end{algorithmic}
\end{algorithm}

For the setting of budget-constrained clustering, the two most relevant
algorithms we are aware of is the algorithm of \cite{SCoaB} (hereby denoted
as \texttt{S\&T}), and the \texttt{IU\_RED} algorithm of \cite{MJ24}. These
algorithms are somewhat similar to our approach, in that they interleave
uniform sampling and a non-uniform sampling scheme. However, the sampling
scheme is very different than ours, and focuses on finding the edge to which
the derivative of the 2nd Laplacian eigenvector is most sensitive. This has
two drawbacks. First, it is specifically designed for spectral clustering and
the case of $k=2$ clusters, which is based on the 2nd Laplacian eigenvector.
Extending this to more than $2$ clusters requires either recursive
partitioning (which can be suboptimal), or considering sensitivity w.r.t.
$k-1$ eigenvectors, and it is not clear what is the best way to do so.
Second, computing eigenvector derivatives requires a full spectral
decomposition at each iteration, which can be quite costly or impractical for
large matrices. In contrast, our algorithm does not compute derivatives.
Therefore, when used  with spectral clustering methods, which require only  the
smallest $2k$ eigenvectors, we have a significant gain.
    
It is possible to
speed up implementation
even further, in the context of spectral clustering. Since only a single edge is added per iteration,
one can use the previously computed eigenvectors as an initial value for fast
iterative eigenvector solvers (although restarting every couple of steps is
advised). Another possible option is to pick several edges from this
distribution at each step, which makes this process parallelizable.

\section{Experiments}

We tested our \texttt{CLUS2K} algorithm on several datasets, and compared it
to the \texttt{S\&T} and \texttt{IU\_RED} discussed earlier (other
alternatives were tested in \cite{SCoaB} and shown to be inferior). It is
important to note that \texttt{S\&T} and \texttt{IU\_RED} were designed
specifically for $k=2$ and spectral clustering using the unnormalized Laplacian 
$L_G$, while we also tested for various values of $k$, and using the
normalized Laplacian $\mathcal{L}_G$  as well \cite{normCut}. The \texttt{IU\_RED} performed badly (perhaps
because it relies substantially on the $k=2$ assumption) in these cases while
\texttt{S\&T} performed surprisingly well (yet still inferior to
\texttt{CLUS2K} ). Clustering was measured by cluster purity. The purity of a
single cluster is the percent of the most frequent class in the cluster. The
purity of a clustering is a weighted average of its single cluster purity,
weighted by the number of elements in each cluster. The purity shown is averaged over 5 runs.
\subsection{Synthetic Data}\label{synthetic}
\begin{figure}[!th]
\centering
\includegraphics[width=0.8\textwidth]{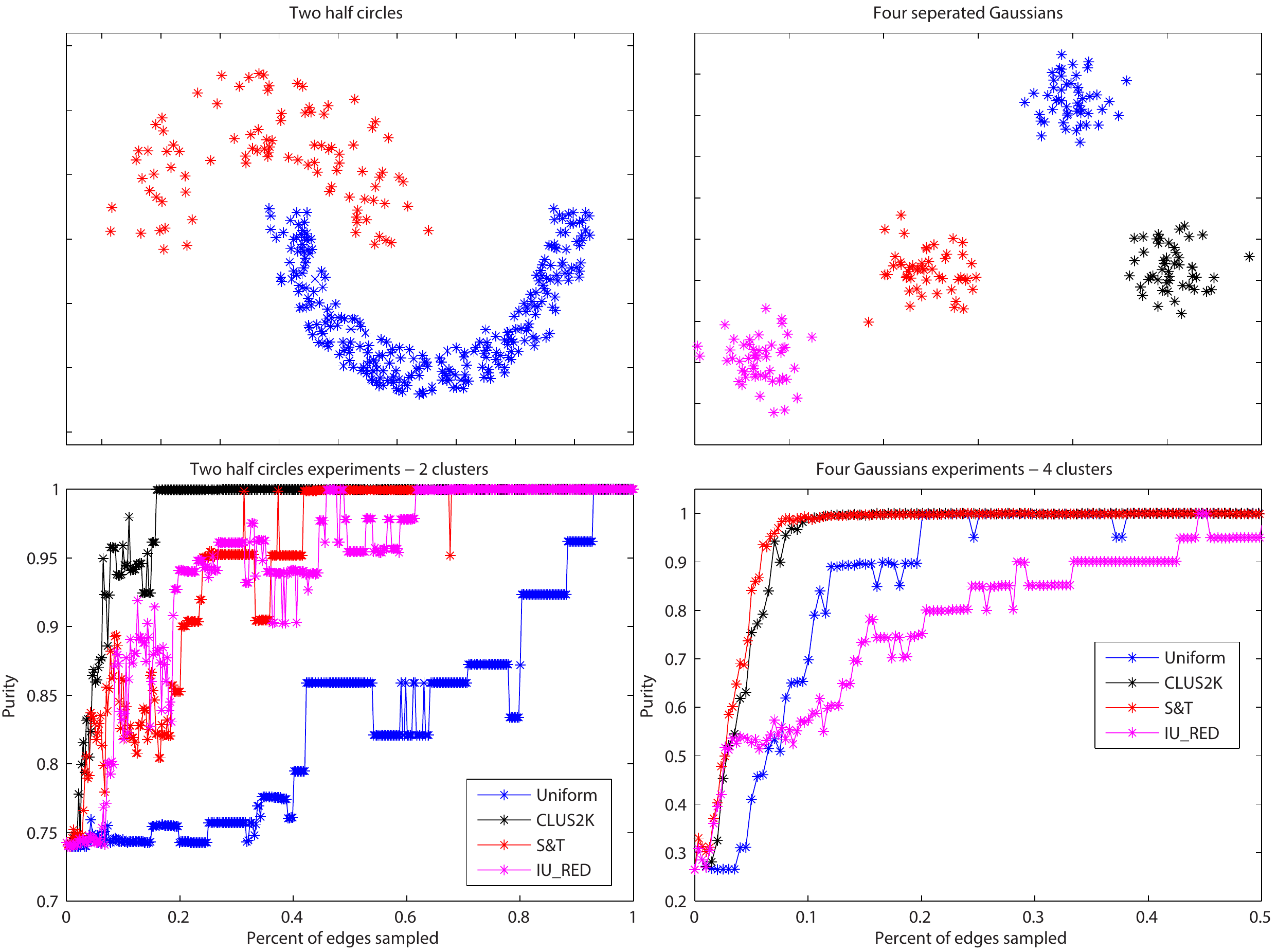}
\caption{Synthetic datasets results}
\end{figure}

The synthetic experiments were performed on two datasets - The two half
circles dataset, and a dataset comprising of four well separated Gaussians,
both experiments used unnormalized spectral clustering (see figure
\ref{synthetic}) using a gaussian weight matrix. The two
half circles is a classic clustering dataset with $k=2$ clusters. The Gaussian
dataset shows how the various algorithms handle an easy $k>2$ dataset, and
\texttt{IU\_RED} performs worse than uniform sampling in this case.

\subsection{Real Data}
We tested on three further datasets - the iris and glass \href{http://archive.ics.uci.edu/ml/\#UCI}{UCI\ datasets}   (both with k\textgreater2 clusters)\  using a gaussian weight matrix, and the Caltech-7 dataset, a subset of the Caltech-101 images datasets with 7 clusters gathered by \cite{caltech7}, using the similarity matrix suggested by \cite{alon}.
We tested each dataset using both the normalized and unnormalized Laplacian for clustering. The results are presented in figure \ref{uci}
\begin{figure}[!th]\label{uci}
\centering
\includegraphics[width=0.8\textwidth]{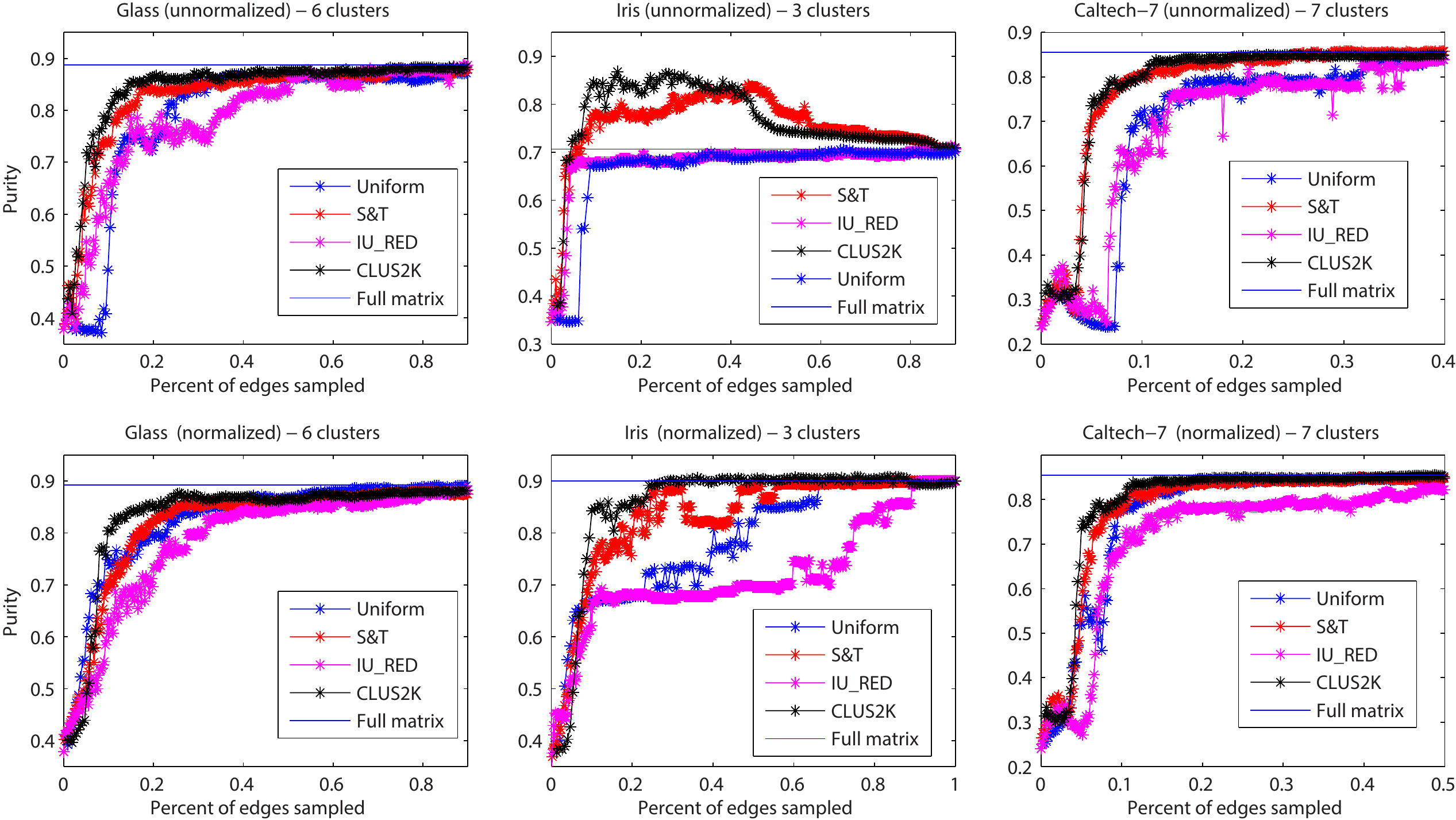}
\caption{UCI datasets results}
\end{figure}\\

Overall, the experiments show that the \texttt{CLUS2K} algorithm performs
as good or better than previous algorithms for budget-constrained clustering,
while being significantly computationally cheaper as well as more general.

\clearpage
%\bibliographystyle{plain}
%\begin{thebibliography}{

%\end{thebibliography}
\clearpage

\appendix
\centerline{\bf{\Large{Supplementary Material}}}
\section{Proof of Theorem \ref{spectralBound}}
We will need to Hoeffding-Chernoff bound for negative dependence:
\begin{theorem} \label{hoeffding}
Assume $X_i\in [0,1]$ are negatively dependent variables and define $X =\sum\limits_{i=1}\limits^n X_i$ then
\begin{equation*}
P\left( |X-\mathds{E}[X]|\geq\epsilon\mathds{E}[X]\right) \leq 2\exp\left(-\frac{\epsilon^2\mathds{E}[X]}{3}\right )
\end{equation*}
\end{theorem}
See \cite{concentration} 1.6 and 3.1 for details.\\

First it is important to note a change in notation from \cite{speilman} in order to be consistent with notation used in \cite{KarSpar}. A $\epsilon$-spectral approximation in our paper is weaker then  a $(1+\epsilon)$-spectral approximation in \cite{speilman}.\\

We will now go over the main changes needed to prove Theorem 6.1 in \cite{speilman} (disregarding S.2) with negatively dependent sampling of edges and weights $w_{ij}\in [0,1]$.  \\

The proof of Claim 6.5 is quite straightforward. The claim of Lemma 6.6 
needs to be changed to $\mathds{E}[\Delta_{r,t}^k\Delta_{t,r}^l]\leq\frac{w_{r,t}}{\gamma^{k+l-1}d_r}$ instead of  $\mathds{E}[\Delta_{r,t}^k\Delta_{t,r}^l]\leq\frac{1}{\gamma^{k+l-1}d_r}$. The changes to the proof are again straightforward (remembering $w_{ij}\in [0,1]$).\\

The main change is to Lemma 6.4. Using the modified Lemma 6.6 and substituting negative dependents for independence one can prove
\begin{equation*}
\sum_{\sigma\,valid\,for\,T,\tau}\prod_{s\in T}\left[\Delta_{v_{s-1},v_s}\prod_{i:\tau(i)=s}\Delta_{v_{i-1},v_i}\right]\leq\frac{1}{\gamma^{k-|T|}}\sum_{\sigma\,valid\,for\,T,\tau}\prod_{s\in T}\frac{w_{v_{s-1},v_s}}{d_{v_{s-1}}}.
\end{equation*} 
instead of equation 10 in the paper. The last change is to pick $\sigma(s)$ proportional to $w_{v_{s-1},v_s}$ instead of uniformly to prove that 

\begin{equation*}
\sum_{\sigma\,valid\,for\,T,\tau}\prod_{s\in T}\frac{w_{v_{s-1},v_s}}{d_{v_{s-1}}}\leq1
\end{equation*}
 instead of equation 11. From there on all changes are straightforward.
\section{Proof of Theorem \ref{specClusterable}\\ }
Let $\tilde{L}=\tilde{L}^{in}+\tilde{L}^{out}$. Let $P$ be the zeros eigenspace of $\tilde{L}^{in}$, which is the same as the zero eigenspace of $L^{in}$, if all the $\tilde{W}^i$ are connected. Let $Q$ be the space spanned by the first $k$ eigenvectors of $\tilde{L}$. According to the Sin-Theta theorem \cite{sin_theta}, $||\sin(\Theta(P,Q))||\leq\frac{||\tilde{L}^{out}||}{\mu^{in}_2}$ where $||\cdot ||$ is the spectral norm of $\tilde{L}^{out}$ and $\mu_2^{in}$ is the second smallest unnormalized eigenvalue of $\tilde{L}^{in}$. To prove the theorem we will show that $\mu_2^{in}= \Omega(n^\alpha)$ and that $||\tilde{L}^{out}|| = \mathcal{O}(n^\beta+n^\gamma)$.\\

The first claim is through using the first two assumptions and the following  lemma

\begin{lemma}\label{nor_unnor}
Let $\lambda_2$ and $\mu_2$ be the second smallest normalized and unnormalized eigenvalues of $L$, and $d = \min_iD_{ii}$ then $\mu_2\geq \lambda_2\cdot d$.
\end{lemma}

\begin{prf}
From the min-max theorem we have 
\begin{equation*}
\mu_2=\min_{U:\,\dim(U)=2}\left\{\max_{x\in U\backslash\{0\}}\frac{x^TLx}{||x||^2}\right\}
\end{equation*}
\begin{equation*}
\lambda_2=\min_{U:\,\dim(U)=2}\left\{\max_{x\in U\backslash\{0\}}\frac{x^TLx}{x^TDx}\right\}
\end{equation*}
and the lemma follows from the fact that $x^TDx\geq d||x||^2$.\qed
\end{prf}

Let $m_1$ be the number of edges needed in order to have an $\epsilon$-spectral approximation of each inner-cluster matrix $W^{i}$ for $\epsilon=3/4$ with probability $\delta/2$, then by theorem \ref{spectralBound} we have that $m_1 = \tilde{\mathcal{O}}(n^{2-\alpha})$. Using this fact, the first two assumptions and lemma \ref{nor_unnor}, it is easy to see that $\mu_2^{in}=\Omega(n^\alpha) $.\\

We now need to show that $||\tilde{L}^{out}|| = \mathcal{O}(n^\beta+n^\gamma)$. The main tool would be the matrix Chernoff inequlity for sampling matrixes without replacements.

\begin{theorem}
Consider a finite sequence of Hermitian matrices ${\bf{X_1}},...,{\bf{X_k}}$ sampled uniformly without replacements from a finite set of matrices of dimension n. Assume that
\begin{equation*}
{\bf{X_k\succeq 0}}  \,\,\,\,\,\,\,\,\,\,\,||{\bf{X_i}||\leq R}.
\end{equation*}  
Define $\bf{Y=\sum\limits_{i=1}\limits^k X_i}$ then 
\begin{equation*}
P(||{\bf{Y}||\geq(1+\epsilon)||\mathds{E}[Y]}||)\leq n\cdot\left(\frac{e^\epsilon}{(1+\epsilon)^{1+\epsilon}}\right)^{||\mathds{E}[Y]||/R}
\end{equation*}
\end{theorem}
\begin{prf}
This is an adaptation of theorem 5.1.1 from $\cite{tropp}$ replacing the independence requirement to sampling without replacements. In order to adapt the proof we notice that the only place where independence is used in in lemma 3.5.1 (subadditivity of the matrix cumulant generating functions) where we need to prove that 
\begin{equation}\label{cgfs}
\forall \theta\in\mathds{R}\,\,\,\,\,\,\,\,\,\,\,\,\mathds{E}\left[Tr\left(\exp\left(\sum\theta{\bf X_i}\right)\right)\right]\leq Tr\left(\exp\left(\sum\log\mathds{E}e^{\theta\bf X_i}\right)\right)
\end{equation}
Using the result of $\cite{sampWithout}$, if $\bf X_i$ are sampled uniformly at random without replacements for a finite set, and $\bf Y_i$ are sampled with the same probability with replacements then 
\begin{equation}
\mathds{E}\left[Tr\left(\exp\left(\sum\theta{\bf X_i}\right)\right)\right]\leq \mathds{E}\left[Tr\left(\exp\left(\sum\theta{\bf Y_i}\right)\right)\right]
\end{equation}
so we can conclude that 
\begin{equation*}
\mathds{E}\left[Tr\left(\exp\left(\sum\theta{\bf X_i}\right)\right)\right]\leq\mathds{E}\left[Tr\left(\exp\left(\sum\theta{\bf Y_i}\right)\right)\right]\leq 
\end{equation*}
\begin{equation*}
\leq Tr\left(\exp\left(\sum\log\mathds{E}e^{\theta\bf Y_i}\right)\right) = Tr\left(\exp\left(\sum\log\mathds{E}e^{\theta\bf X_i}\right)\right).
\end{equation*}
where the second inequality is from \ref{cgfs} as $\bf{Y_i}$ are independent.
\end{prf}

We define for each edge $e$ connecting nodes in different clusters the matrix $\bf X_e$ that is equal to zero with probability $1-p$ and is equal to $\frac{1}{p}L_e$ with probability $p$, where $L_e$ is the Laplacian of a single edge graph with weight $w_e$. Then $\tilde{L}^{out}=\sum\limits_{e\in S_{out}}\bf X_e$, $\,\,\mathbb{E}[\tilde{L}^{out}]=L^{out}$, $\,\,\bf X_e\succeq0$ and $||{\bf X_e||\leq}1/p$.

If we use the matrix Chernoff inequality with $1+\epsilon=2e\cdot n^{\gamma-\beta}$  then 
\begin{equation*}
P(||\tilde{L}^{out}||\geq2en^\gamma)\leq n\left(\frac{1}{2n^{\gamma-\beta}}\right)^{2en^\gamma p}
\end{equation*}
So if $p=m/\binom{n}{2}=\mathcal{O}\left(\frac{log(n)}{n^\gamma}\right)$ we get that $P(||\tilde{L}^{out}||\geq2en^\gamma)
<\delta/2$ for large enough $n$.
\section{Cut Approximation}

 We will start by proving an analog of theorem \ref{spectralBound} in the paper. We will use the following lemma from \cite{KarLemma}:

\begin{lemma} \label{combLemma}
Let $G$ be an undirected graph with $n$ vertices and minimal cut $c>0$. For all $\alpha\geq1$ the number of cuts with weight smaller of equal to $\alpha c$ is less then $n^{2\alpha}$.
\end{lemma}

The lemma is proven in \cite{KarLemma} for graphs with integer weights, but the extension to any positive weights is trivial by scaling and rounding. We can now state and prove the theorem guaranteeing good cut approximations.

\begin{theorem}\label{origKar}
Let $G$ be a graph with weights $w_{ij}\in[0,1]$, with minimal cut $c>0$, and $\tilde{G}$ its approximation after sampling $m$ edges uniformly.  If $m\geq\binom{n}{2}\frac{3(2\ln(n)+\ln(\frac{1}{\delta})+k)}{\epsilon^2c}$ where $k= \ln(2+4\ln(n))$, then the probability that $\tilde{G}$ is not an $\epsilon$-cut approximation is smaller then $\delta$.
\end{theorem} 

\begin{prf}
This is an adaptation of the proof in \cite{KarSpar} - consider a cut with weight $\alpha c$. Let $p=m/\binom{n}{2}$ the probability to sample a single edge. Let $Y_e = X_e \cdot w_e$ where $X_e$ is an indicator whether edge $e$ on the cut was sampled and $w_e$ its weight. Define $Y$ the sum of $Y_e$ on all the edges along the cut, then by the fact that edges are negatively dependent and theorem \ref{hoeffding}, the probability that the cut is not an $\epsilon$ approximation is smaller then
\begin{equation*}\begin{split}
2\exp\left(-\frac{\epsilon^2\mathds{E}[Y]}{3}\right)=2\exp\left (-\frac{\epsilon^2\alpha c p}{3}\right)\leq
\\ \leq 2\exp\left( -\left(\ln\left(\frac{1}{\delta}\right)+k\right)\alpha\right)\cdot n^{-2\alpha}
\end{split}\end{equation*}
Define $P(\alpha) =  2\exp\left( -\left(\ln\left(\frac{1}{\delta}\right)+k\right)\alpha\right)\cdot n^{-2\alpha}$ and let $f(\alpha)$ the number of cuts with value $\alpha c$ in the original graph. By the union bound the probability that some cut is not an $\epsilon$ approximation is less then $\sum\limits_{\alpha\geq 1}f(\alpha)P(\alpha)$ (notice that this sum is well defined since $f(\alpha)$ is non zero only in a finite number of $\alpha$ values). Defining $F(\alpha) = \sum\limits_{\alpha\geq x}f(x)$ then by the previous lemma $F(\alpha) \leq n^{2\alpha}$. Let $g$ be any measure on $[1,\infty)$ such that $G(\alpha)=\int\limits_1\limits^{\alpha}dg\leq n^{2\alpha}$, then the integral $\int\limits_1\limits^{\infty}P(x)dg$ is maximized when $G(\alpha)=n^{2\alpha}$. This is due to the fact that $P$ is a monotonically decreasing function, so if the inequality is not tight at some point $x_1$ we could increase the value by picking $\tilde{g}(x) = g(x) +\tilde{\epsilon} \delta(x-x_1) -\tilde{\epsilon}\delta(x-x_2)$ for some appropriate $x_2>x_1$ and $\tilde{\epsilon}$ (where $\delta$ is the Dirac delta function) . From this we can conclude that the probability of some cut not being an $\epsilon$-approximation is bounded by

\begin{equation*}\begin{split}
& n^2P(1)+\int\limits_1\limits^{\infty}P(\alpha)\frac{dn^{2\alpha}}{d\alpha}d\alpha = 2\delta e^{-k}+\frac{4\delta\ln(n)e^{-k}}{\ln(\frac{1}{\delta})+k}
\\
&\leq\delta\left(2+4\ln(n) \right)\exp(-k) = \delta
\end{split}\end{equation*}\qed
\end{prf}
 A drawback is that the theorem gives a bound that depends on the minimal cut, which we do not know, and unlike the situation in \cite{KarSpar} we cannot approximate it using the full graph. We can prove a bound that uses only known data about the graph. The following theorem shows we can lower bound $c$.

\begin{lemma}\label{cmin}
Let $G$ be a graph with weights $w_{ij}\in[0,1]$, with minimal cut $c$, and $\tilde{G}$ its approximation after sampling $m$ edges with minimal cut $\tilde{c}>0$. Define $p = {m}/{\tbinom{n}{2}}$ the probability to sample a single edge. Also define $l = \frac{3\ln(\frac{1}{\delta})}{4}$ and $\beta = \sqrt{1+\frac{l}{p\tilde{c}}}-\sqrt{\frac{l}{p\tilde{c}}}$. With probability greater then $1-\delta$ the following inequality holds - $c\geq \tilde{c}\cdot \beta^2$.
\end{lemma}

\begin{prf}
Let $S$ be a subset of vertices such that $|\partial_G S| = c$ then from the Chernoff-Hoeffding inequality (the one-sided version)
\begin{equation*}\begin{split}
&P\left(|\partial_{\tilde{G}}S| \geq(1+ \epsilon)  |\partial_GS|\right) = P\left(p|\partial_{\tilde{G}}S| \geq(1+ \epsilon) p |\partial_GS|\right)
\\ &  \leq \exp\left(-\frac{\epsilon^2pc}{3}\right)
\end{split}\end{equation*}
Where we multiply by $p$ to have all the elements bounded by $1$.
Setting $\epsilon =\sqrt{ \frac{3\ln(\frac{1}{\delta})}{pc}}$ we get that with probability greater then $1-\delta$ that \\$pc\left( 1+\sqrt{ \frac{3\ln(\frac{1}{\delta})}{pc}}\right)=pc+\sqrt{3\ln(\frac{1}{\delta})}\sqrt{pc}\geq |\partial_{\tilde{G}}S|\geq p\tilde{c}$. By completing the square we get that 
\begin{equation*}
\left( \sqrt{cp}+\sqrt{\frac{3\ln(\frac{1}{\delta})}{4}}\right)^2 = \left( \sqrt{cp}+\sqrt{l}\right)^2\geq p\tilde{c}+l
\end{equation*}
which means (after some simple algebraic manipulation) that 
\begin{equation*}
c\geq \tilde{c}\beta^2
\end{equation*}\qed
\end{prf}
We can combine these to theorems and get

\begin{theorem}
Let $G$ be a graph with weights $w_{ij}\in[0,1]$ and $\tilde{G}$ its approximation after sampling $m$ edges with minimal cut $\tilde{c}>0$. Define  $\beta$ and $k$ as in previous theorems.  If $m\geq \binom{n}{2}\frac{3(2\ln(n)+\ln(\frac{2}{\delta})+k)}{\epsilon^2\beta^2\tilde{c}}$ then the probability that $\tilde{G}$ is not an $\epsilon$-cut approximation is smaller then $\delta$.
\end{theorem}

\begin{prf}
This is just using lemma \ref{cmin} with error probability $\frac{\delta}{2}$ and using that $c$ for theorem \ref{origKar} with the same error probability and the union bound. \qed
\end{prf}

This theorem gives a high probability bound that depends only on observable quantities. While the notation is a bit cumbersome, it is easy to see that if $p\tilde{c} \gg \ln(\frac{1}{\delta}) $, i.e. the unscaled weight of the smallest cut is not too small, then $\beta \approx 1$ and we have a bound that is almost as good as if we knew the real $c$.\\

We will now prove theorem \ref{clustering} in the paper.
 
\begin{trm}[\ref{clustering}]
Let $G$ be a graph with weights $w_{ij}\in[0,1]$ and $\tilde{G}$ its approximation after observing $m$ edges. Assume $G$ is partitioned into $\ell$ clusters each has minimal cut greater or equal to $c_{in}$, and the cuts separating clusters from the others is smaller then $c_{out}$. Furthermore assume $c_{in}>4c_{out}$. If $m\geq\frac{12n^2}{c_{in}}\left(2\ln(n)+\ell\ln(\frac{2}{\delta})+k\right)$ then the cuts separating the clusters are smaller then any cut that cuts into one of the clusters.
\end{trm}
\begin{prf}
After seeing $m$ edges, the probability for sampling any edge inside any cluster is $p={m}/{\tbinom{n}{2}}$. By theorem \ref{origKar} we have that if $m\geq\frac{12n^2}{c_{in}}\left(2\ln(n)+\ln(\frac{2^\ell}{\delta})+k\right)$ then the probability of any cut in a single cluster being smaller then $\frac{c_{in}}{2}$ is smaller then $\frac{\delta}{2\ell}$, with the union bound we have that with probability greater then $1-\frac{\delta}{2}$ all cuts in any cluster (and therefore any cut in $\tilde{G}$ that cuts some cluster) have weights greater or equal to $\frac{c_{in}}{2}$.\\

We now need to show that the cuts separating the clusters are not too large. Consider a cut separating some clusters from the others. If the weight of this cut is $c$ we need to show that with probability greater then $1-\frac{\delta}{2^{\ell+1}}$ we have $\tilde{c}<\frac{c_{in}}{2}$. This means that we want to show that $\tilde{c}<(1+\tilde{\epsilon})c\leq(1+\tilde{\epsilon})c_{out}=\frac{c_{in}}{2}$, i.e. we can use the negatively dependent Chernoff-Hoeffding inequality (theorem \ref{hoeffding}) with $\tilde{\epsilon}=\frac{c_{in}}{2c_{out}}-1>\frac{c_{in}}{4c_{out}}$ (using the fact that $c_{in}>4c_{out}$) and get that the $P(p\tilde{c}-pc>(1+\tilde{\epsilon})pc_{out})\leq\exp\left(-\frac{\tilde{\epsilon}^2pc_{out}}{3}\right)\leq \exp\left(-\frac{pc_{in}}{12}\right)$. As $m\geq\frac{12n^2}{c_{in}}\ln(\frac{2^\ell}{\delta})$ we can finish the proof. \qed
\end{prf}

\section{Adaptive Sampling with replacement}
%\$todo tod !! I took what I wrote on adaptive sampling (with proof) and stuck it in the appendix. I think we can still keep that proof in the appendix !! todo todo\\

%It has been shown, in the context of graph sparsification, that by sampling with probability proportional to the connectivity, strong connectivity, effective resistance etc. (!! cite) one can get a good approximation. We cannot do this, as the graph is unknown, but it would seem wise to use an adaptive algorithm that probes with higher probability in areas that so far are less connected (or by some other scheme that comes to mind). \\

While we found that adaptive sampling with replacements did not perform as well as without replacements in practice, for completeness we will present her a proof that it has the same theoretical  guarantees as uniform sampling for cut approximation.

Let $\tilde{G}_i$ be the graph build at step $i$, an adaptive sampling algorithm is an algorithm who picks an edge at step $i+1$ with probability $p(e;\tilde{G}_i)$ that depends on $\tilde{G}_i$. In order to prove that with high probability $\tilde{G}=\tilde{G}_m$ is a $\epsilon$-approximation of $G$ for $m=o(n^2)$ we need that $p(e;\tilde{G}_i)$ isn't too small on any edge. This can be easily done by sampling according to a modified distribution - with probability $0.5$ pick an edge uniformly, and with probability $0.5$ pick it according to $p(e;H_i)$. The new distribution satisfies $\tilde{p}(e;\tilde{G}_i) = \frac{1}{2}p(e;\tilde{G}_i)+\frac{1}{n(n-1)}>\frac{1}{n^2}$.\\

%\begin{algorithm}[H]
%\SetKwInOut{Input}{Input}
%\SetKwInOut{Output}{Output}
%\Input{size of graph $n$, budget $m$, graph dependent distribution $p(e;H)$}%

%$\tilde{G}_0 \leftarrow 0 $ \tcp{graph with all zero weights}
%\For{$i=1:m$}{
%       $\tilde{G}_i = \tilde{G}_{i-1}$\;
%        $\tilde{p}(e) = \frac{1}{2}p(e;\tilde{G}_{i-1})+\frac{1}{n(n-1)}$\;
%       $e \leftarrow $select random edge according to $\tilde{p}$\;
%        $w(e) \leftarrow$ weight of edge $e$ in $G$\;
%        add $\frac{1}{\tilde{p}(e)}w(e)$ to the weight of edge $e$ in %$\tilde{G}_i$\;
%      }
%%$\tilde{G} \leftarrow \frac{1}{m}\tilde{G}_m$\;
%\Output{$\tilde{G}$}
%\caption{Generic Bounded Adaptive Sampling Algorithm (BASA)}
%\end{algorithm}

The graphs $\tilde{G}_i$ are by no means independent. Although one can view (after subtracting the mean) them as a martingale process, using the method of bounded differences \cite{concentration} will not suffice, as it depends on the square of the bounding constant, so we will have a $n^4$ factor that only gives a trivial bound. We will next show that a high probability bound does exists. \\

Consider a cut with weight $c$ that contains the edges $e_1,...,e_l$ and consider any bounded adaptive sampling algorithm with replacements with $m$ steps. Define $X_{ik}$ with $1\leq i \leq l$ and $1\leq k \leq m$ to be the random variable that has value $\frac{w(e_i)}{\tilde{p}(e_i)}$ if the edge $e_i$ was chosen at step $k$ and zero otherwise. Define $Y_k=\sum\limits_{i=1}\limits^l X_{ik}$, $Y_k$ is the weight added to the cut at step $k$ and its expectation is $c$.

\begin{lemma}\label{boundLemma}
If $\forall i,l:\, \tilde{p}(e_i)\geq \rho$ and $w(e_i)\leq 1$ then 
\begin{equation*}
\mathds{E}[\exp(t\rho Y_k)|\tilde{G}_{k-1}] \leq \exp(c\rho(e^t-1))
\end{equation*}
\end{lemma}
\begin{prf}
Since at most one of the positive variables $X_{ik}$  is nonzero for a constant $k$ then they are negatively dependent when conditioned by $\tilde{G}_{k-1}$. This implies that  $\mathds{E}[\exp(t\rho Y_k)|\tilde{G}_{k-1}] \leq\prod\limits_{i=1}\limits^l \mathds{E}[\exp(t\rho X_{ik})|\tilde{G}_{k-1}] $. By definition of  $X_{ik}$ we get that 
\begin{equation}\label{temp1}
\mathds{E}[\exp(t\rho X_{ik})|\tilde{G}_{k-1}] = \tilde{p}(e_i)\cdot\exp\left(\frac{t\rho w(e_i)}{\tilde{p}(e_i)}\right)+(1-\tilde{p}(e_i))
\end{equation}
 One can easily verify that the right hand side of equation \ref{temp1} decreases monotonically with $\tilde{p}(e_i)$, so the fact that $\rho <\tilde{p}(e_i)$ and $w(e_i)\leq 1$ implies that
\begin{equation*}\label{temp2}
\mathds{E}[\exp(t\rho X_{ik})|\tilde{G}_{k-1}] \leq \rho w(e_i)e^t+(1-\rho w(e_i))=
\end{equation*}
\begin{equation*}
= \rho w(e_i) (e^t-1) +1 \leq \exp(\rho w(e_i)(e^t-1) )
\end{equation*}
Where the last inequality is due to the fact that for $1+x<e^x$. We can finish the proof since 
\begin{equation*}\begin{split}
&\mathds{E}[\exp(t\rho Y_k)|\tilde{G}_{k-1}] \leq\prod\limits_{i=1}\limits^l \mathds{E}[\exp(t\rho X_{ik})|\tilde{G}_{k-1}] \\
&\leq \exp(\rho c (e^t-1)).
\end{split}\end{equation*}
as $\sum w(e_i) = c$.\qed
\end{prf}
We can now prove the concentration of measure bound for a single cut

\begin{theorem}
Let $G$ be a graph such that $w(e_i)\leq 1$ and $\tilde{G}=\tilde{G}_m$ the output of a bounded adaptive sampling algorithm with replacements such that $\tilde{p}(e_i)\geq \rho$ then the probability that a cut with weight $c$ in $\tilde{G}_m$ is not a $\epsilon$-approximation is bounded by $2\exp\left(-\frac{\epsilon^2\rho m c}{3}\right)$.
\end{theorem}
\begin{prf}
We need to show that 
\begin{equation*}
P\left(\left|\sum_{k=1}^mY_k-mc\right|>\epsilon m c\right)\leq 2\exp\left(-\frac{\epsilon^2\rho m c}{3}\right)
\end{equation*}
The proof is similar to the proof of the Chernoff bound, replacing independence with lemma \ref{boundLemma}. First look at $P\left(\sum_{k=1}^mY_k>(1+\epsilon) m c\right)$. Using the standard trick for all $t>0$
\begin{equation*}\begin{split}
&P\left(\sum_{k=1}^mY_k>(1+\epsilon) m c\right)=\\
&P\left(\exp\left(t\rho\sum_{k=1}^mY_k\right)>\exp(t(1+\epsilon) \rho m c)\right)
\end{split}\end{equation*}
By the Markov inequality this is bounded by $\frac{\mathds{E}\left[\exp\left(t\rho\sum\limits_{k=1}\limits^m Y_k \right)\right]}{\exp(t(1+\epsilon) \rho m c)}$. The law of total expectation states that 
$\mathds{E}\left[\exp\left(t\rho\sum\limits_{k=1}\limits^m Y_k \right)\right]=\mathds{E}\left[\mathds{E}\left[\exp\left(t\rho\sum\limits_{k=1}\limits^m Y_k \right)|\tilde{G}_{m-1}\right]\right]$. As $\sum\limits_{k=1}\limits^{m-1}Y_{k}$ is a deterministic function of $\tilde{G}_{m-1}$ this is equal to 
\begin{equation*}\begin{split}
&\mathds{E}\left[\mathds{E}\left[\exp(t\rho Y_m)|\tilde{G}_{m-1}\right]\exp\left(t\rho\sum\limits_{k=1}\limits^{m-1} Y_k \right)\right] \\
& \leq \mathds{E}\left[\exp\left(t\rho\sum\limits_{k=1}\limits^{m-1} Y_k \right)\right]\exp(\rho c (e^t-1)).
\end{split}\end{equation*}
using lemma \ref{boundLemma}. By induction we can conclude that the expectation is smaller then $\exp(\rho mc (e^t-1))$. We have shown that 
\begin{equation*}
P\left(\sum_{k=1}^mY_k>(1+\epsilon) m c\right)\leq\frac{\exp(\rho mc (e^t-1))}{\exp(t(1+\epsilon) \rho m c)}
\end{equation*}
Following the steps as in the standard Chernoff bound proof one can show that this is smaller (for the right $t$) then $\exp\left(-\frac{\epsilon^2\rho m c}{3}\right)$. The proof for this bound on $P\left(\sum_{k=1}^mY_k<(1-\epsilon) m c\right)$ is done in a similar fashion, and using the union bound we finish our proof.
\end{prf}

Using $\rho = \frac{1}{n^2}$ one can now show similar theorems to what we shown in the previous section with this theorem replacing the (negatively dependent) Chernoff bound.
\section{Adaptive Sampling without Replacements}
For a specific graph one can always design a bad biased sampling scheme. Consider an adversarial scheme that always samples the largest weight edge between two constant clusters, it is easy to see that this can lead to bad cut clustering. To circumvent this we will consider graphs where the edge weights between the clusters, which we regard as noise, are picked randomly. 

\begin{assum}\label{assAtapCut1}
Assume $G$ can be partitioned into $k$ clusters of size $\Omega(n)$, within which the minimal cut
is at least $c_{in}=\Omega(n^\alpha)$.
\end{assum}
\begin{assum}
 Assume that the weights of edges between the clusters are $0$, besides $c_{out}=o(n^\alpha)$  edges  chosen uniformly at randomly (without replacement) between any two clusters that have weight $1$.
\end{assum}

\begin{theorem}
Let $\tilde{G}$ be the graph after sampling  $m=\tilde{\Omega}\left(n^{2-\beta}k\ln(\frac{1}{\delta})\right)$ edges  without replacements (with probability $1/2$ of sampling uniformly) with $\beta<\alpha$. Let $\tilde{c}_{in} $ and $\tilde{c}_{out}$ be the minimal cut weight inside any cluster and the maximal cut weight between clusters, under previous assumptions the probability that $\tilde{c}_{in}<\tilde{c}_{out}$ is smaller then $\delta$.
\end{theorem}
\begin{prf}
Using theorem \ref{origKar} on the edges sampled uniformly (remembering that the biased sampling can only increase the cut weight) we get that with probability greater then $\delta/2$, $\tilde{c}_{in}=\tilde{\Omega}\left(\frac{m}{n^2}c_{in}\right)=\tilde{\Omega}(n^{\alpha-\beta})$.  If we consider the weight of any cut between clusters, then the key observation is that because the edges are picked uniformly at random, then whatever the algorithm does is equivalent to running a uniform sampling of a constant edge set. We then get that the expected minimal cut weight is $\tilde{\mathcal{O}}(\frac{m\cdot c_{out}}{n^2})=o(n^{\alpha-\beta})$ using lemma \ref{interLemma} (the upper bound is by looking as if all edges where picked from this cut). We can now use the Markov inequality to show $P(\tilde{c}_{out}/\tilde{c}_{in}<1)=\frac{o(n^{\alpha-beta})}{\Omega(n^{\alpha-\beta})}<\delta/2$. 
\end{prf}
It is simple to generalize this theorem to any uniform weighting that has $o(c_{out})$ expected cut weights.
\clearpage

\end{document}